\documentclass{article}

\usepackage{arxiv}

\usepackage{amsmath,amssymb,amsfonts,amsthm}
\usepackage{algorithmic}
\usepackage{algorithm}
\usepackage[utf8]{inputenc} 
\usepackage[T1]{fontenc}    
\usepackage{hyperref}       
\usepackage{url}            
\usepackage{booktabs}       
\usepackage{nicefrac}       
\usepackage{microtype}      
\usepackage{lipsum}		
\usepackage{graphicx}
\usepackage{cite}
\usepackage{textcomp}
\usepackage{stfloats}
\usepackage{verbatim}
\usepackage{doi}

\newtheorem{theorem}{Theorem}[section]
\newtheorem{lemma}[theorem]{Lemma}

\title{Control-Theoretic Techniques for Online Adaptation \\ of Deep Neural Networks in Dynamical Systems}

\author{Jacob G.~Elkins \\
	Mechanical and Aerospace Engineering Department\\
	The University of Alabama in Huntsville\\
	Huntsville, AL, USA \\
	\texttt{jacob.elkins@uah.edu} \\
	\And
	Farbod Fahimi \\
	Mechanical and Aerospace Engineering Department\\
	The University of Alabama in Huntsville\\
	Huntsville, AL, USA \\
	\texttt{farbod.fahimi@uah.edu} \\
}

\renewcommand{\headeright}{}
\renewcommand{\undertitle}{Preprint}
\renewcommand{\shorttitle}{Control-Theoretic Techniques for Online Adaptation of Deep Neural Networks in Dynamical Systems}

\hypersetup{
pdftitle={Control-Theoretic Techniques for Online Adaptation of Deep Neural Networks in Dynamical Systems},
pdfsubject={ArXiV Preprint},
pdfauthor={Jacob G.~Elkins, Farbod Fahimi},
pdfkeywords={deep learning, neural networks, online learning, control theory},
}

\begin{document}
\maketitle

\begin{abstract}
Deep neural networks (DNNs), trained with gradient-based optimization and backpropagation, are currently the primary tool in modern artificial intelligence, machine learning, and data science. In many applications, DNNs are trained offline, through supervised learning or reinforcement learning, and deployed online for inference. However, training DNNs with standard backpropagation and gradient-based optimization gives no intrinsic performance guarantees or bounds on the DNN, which is essential for applications such as controls. Additionally, many offline-training and online-inference problems, such as sim2real transfer of reinforcement learning policies, experience domain shift from the training distribution to the real-world distribution. To address these stability and transfer learning issues, we propose using techniques from control theory to update DNN parameters online. We formulate the fully-connected feedforward DNN as a continuous-time dynamical system, and  we propose novel last-layer update laws that guarantee desirable error convergence under various conditions on the time derivative of the DNN input vector. We further show that training the DNN under spectral normalization controls the upper bound of the error trajectories of the online DNN predictions, which is desirable when numerically differentiated quantities or noisy state measurements are input to the DNN. The proposed online DNN adaptation laws are validated in simulation to learn the dynamics of the Van der Pol system under domain shift, where parameters are varied in inference from the training dataset. The simulations demonstrate the effectiveness of using control-theoretic techniques to derive performance improvements and guarantees in DNN-based learning systems.
\end{abstract}


\section{Introduction}
Deep neural networks (DNNs) are data-driven, biologically-inspired nonlinear function approximators. DNNs are currently the primary tool in modern artificial intelligence, machine learning, and data science; driving popular breakthroughs in reinforcement learning \cite{mnih2015human}, natural language processing \cite{gpt4}, and content generation \cite{dalle2}. DNNs have enjoyed success due to their generality, ability to learn intricacies from data, and ease of implementation for parallel computation. The modeling and function approximation capability of DNNs has recently exploded, largely due to advancements in computer hardware and increased data generation.

Most commonly, DNNs are trained on an offline dataset via gradient-based optimization and backpropagation \cite{RHW, werbos}. These DNNs are then deployed online for inference, with the DNN model's parameters static, only optimized during the pre-deployment training. If the underlying process that generates the DNN training data changes, known as concept drift or domain shift, DNN performance can suffer \cite{shift_paper, shift_book}. It is desirable for DNN-based models to improve their performance over time during deployment in the real-world, learning from the novel data being processed by the DNN model. However, efficiently updating DNN model parameters from novel signals to guarantee improved performance is not straightforward. Retraining the entire DNN model on only the novel data can introduce catastrophic forgetting, with the model parameters over-optimizing to the recent inputs. However, retraining the entire DNN model with the fully updated dataset can become computationally expensive and inefficient at each datapoint \cite{online_learning_survey}. Relevant problem areas in deep learning include transfer learning, domain adaptation, and domain generalization \cite{transfer_survey, domain_adapt_bounds, pretrain_for_domain_adapt, domain_gen}.

In many DNN applications, data is processed by the deployed DNN model sequentially in time, especially when the DNN is approximating a dynamical quantity or system. Common examples of online DNN deployment include forecast models \cite{forecast_survey} and policies trained with reinforcement learning \cite{sim2real1}. In these examples, there are no intrinsic properties on the DNN trained with gradient-based backpropagation, such as boundedness of DNN outputs or guarantees of convergence. 

The motivating example for this work is the sim2real transfer of reinforcement learning (RL) policies, particularly when policies are used as controllers \cite{sim2real1, sim2real2}. RL-based control policies are typically trained using a simulation of the real-world control problem, where the simulation is constructed to model the real-world physics as accurately as possible. Once desirable performance in simulation is achieved, the policy is then deployed onto the real-world control problem \cite{sim2real_paper}. However, the real-world is highly nonlinear, complex, and difficult to model; and obtaining dynamic models for use in simulation can be expensive \cite{sys_id_book}. A discrepancy between the simulated problem and the real-world problem represents a shift in the training data distribution and the test data distribution; and standard feedforward DNNs trained with conventional RL are not designed to adequately adapt to this distribution shift, resulting in degraded policy performance in the real-world. Common methods for solving the sim2real ``reality gap" include domain randomization \cite{domain_random1, domain_random2, domain_random3}, enforcing robustness via adversarial training \cite{arl1, arl2}, and meta-learning \cite{meta_learning, sim2real_meta}.

Considering the important applications of DNNs to approximating dynamical systems, it is desirable to study how to train and implement DNNs to maximize online performance. In this work, we propose using techniques from control theory to evolve DNNs online during deployment to improve performance online. The mathematical rigor of control theory can be used to establish desirable properties on DNN outputs, such as performance bounds and convergence guarantees. When a DNN of arbitrary depth is deployed to control or predict a dynamical system, the DNN itself can be treated as a dynamical system. Assuming the training distribution is reasonably close to the target distribution, the output of the next-to-last layer forms a basis for function approximation of the last layer. Similar to the distribution shift and transfer learning approaches which retrain the last layer while freezing upper layers in the DNN \cite{labonte, spottune, lastlayer1, lastlayer2}, we evolve the last layer of the DNN to learn online during operation in a provably-stable manner, shown in the block diagram in Figure~\ref{fig:onnst_block_diag}. The novel update rule proposed in this work is based on the super-twisting algorithm (STA) \cite{levant_st, shtessel2014sliding, moreno_st}. The STA is used in control, observation, and differentiation, when finite-time asymptotic convergence is desired via a continuous control law \cite{shtessel2014sliding}. Using the STA, we further show that other deep learning techniques can be combined with our online adaptation law to improve performance, such as spectral normalization during training \cite{bartlett_sn, miyato_dnn, miyato_gan}. We show that, when the time derivative of the DNN input vector is noisy (from measurement or numerical differentiation), spectrally normalizing the DNN during training defines the upper bound of the convergence error.

\section{Motivating Example}

Consider the unforced Van der Pol equation \cite{khalil}, with dynamics written as

\begin{equation}
\label{eq:vanderpol}
\begin{split}
    \dot{z} &= \epsilon \left( z - \frac{1}{3}z^3 - \theta \right) \\
    \dot{\theta} &= z
\end{split}
\end{equation}

\noindent where $\epsilon$ is a constant parameter. The goal is to use a DNN to learn the unknown dynamics of $z$, assuming the dynamics of $\theta$ are known, and $z$, $\theta$, $\dot{z}$ can be measured. Training a DNN on a feature dataset $\mathcal{X} = \{z, \theta\}$ and a label dataset of $\mathcal{Y} = \{ \dot{z} \}$, generated by the nominal system of Equation~\eqref{eq:vanderpol} with $\epsilon = 1$, we can approximate $\dot{x}$ with a DNN to arbitrary accuracy, shown in Figure~\ref{fig:dnn_nom}.

\begin{figure}[!htbp]
\centering
\includegraphics[width=3in]{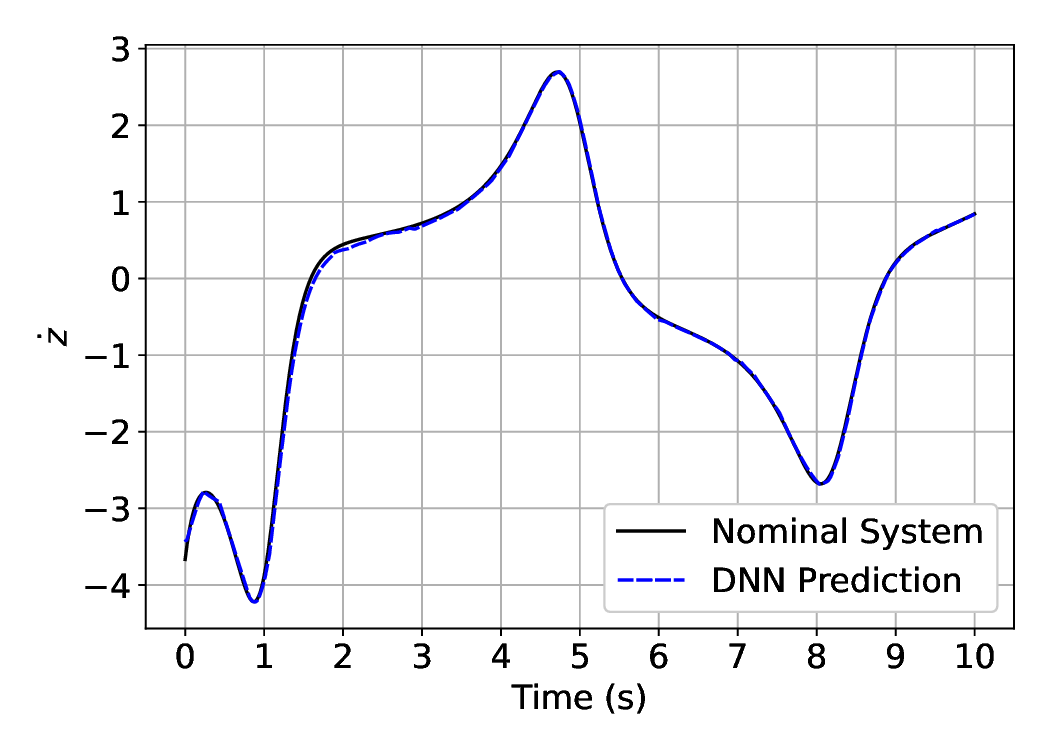}
\caption{DNN predictions on the nominal system ($\epsilon = 1$).}
\label{fig:dnn_nom}
\end{figure}

\noindent However, assuming a ``real" system of Equation~\eqref{eq:vanderpol} with $\epsilon = 1.5$, the DNN experiences domain shift, with the testing data (real system) generated from a different distribution than the training data (nominal system). Without any online learning or retraining, the DNN performance suffers, shown in Figure~\ref{fig:dnn_real}.

\begin{figure}[!htbp]
\centering
\includegraphics[width=3in]{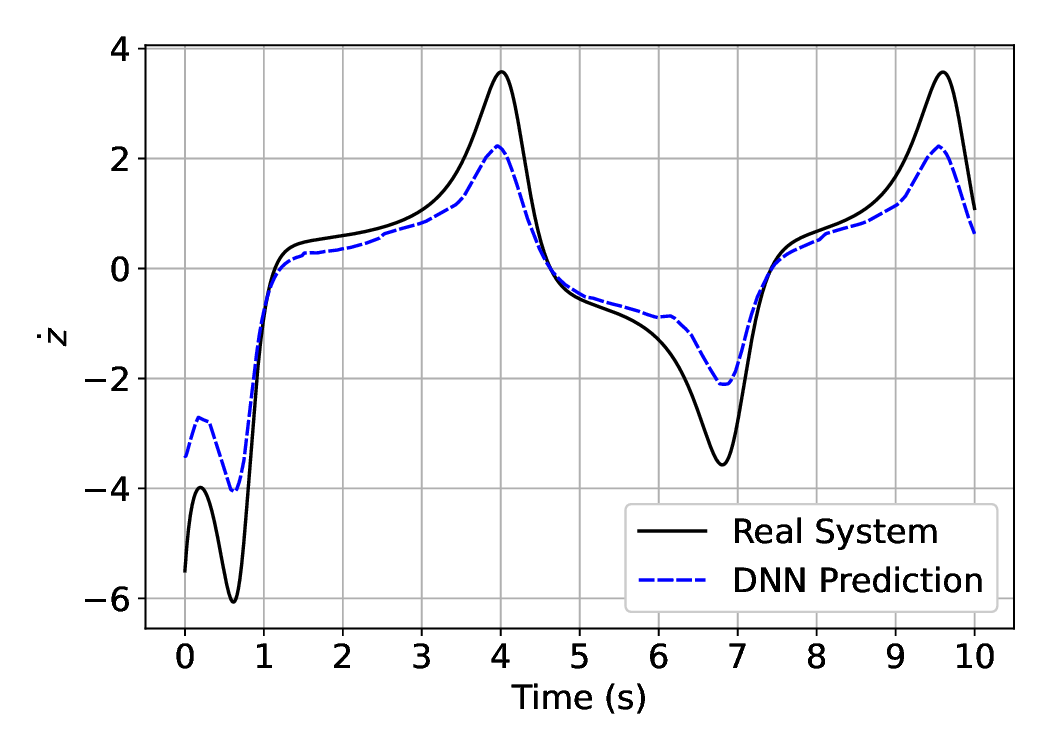}
\caption{DNN predictions on the real system ($\epsilon = 1.5$).}
\label{fig:dnn_real}
\end{figure}

\noindent Since the DNN is being trained to approximate a dynamical system in \eqref{eq:vanderpol}, we can consider the DNN itself as a dynamical system. Considering the DNN as a dynamical system to control allows the use of control theory to update the DNN online to achieve desirable performance on the real system without full retraining, shown in Figure~\ref{fig:dnn_online}.

\begin{figure}[!htbp]
\centering
\includegraphics[width=3in]{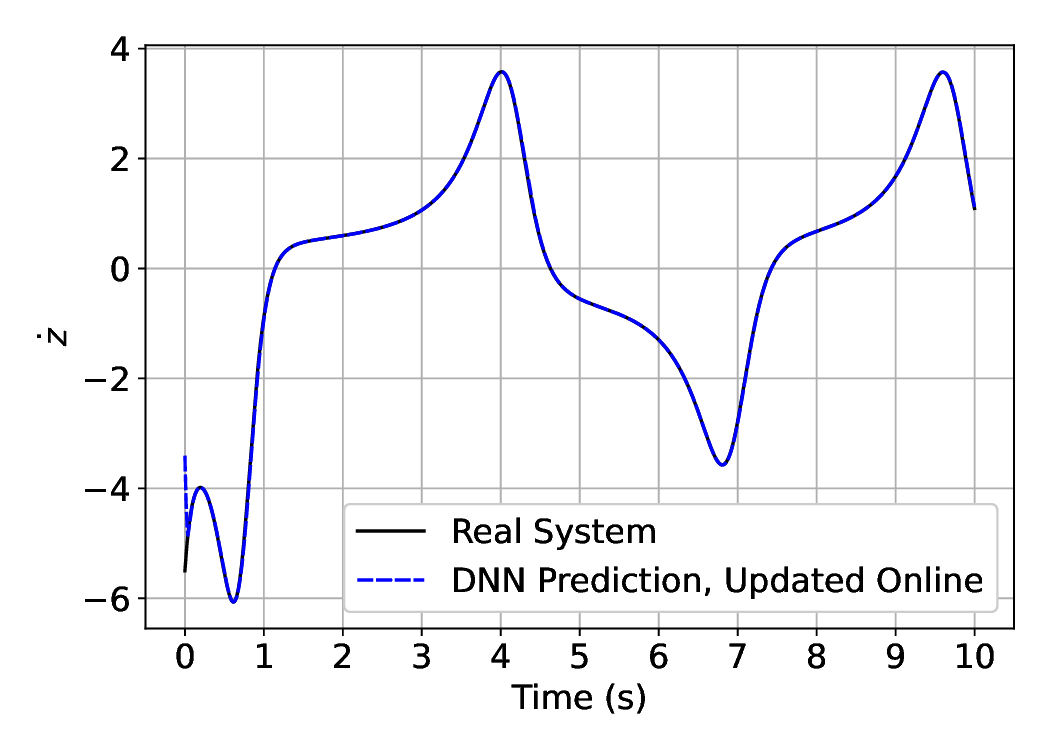}
\caption{Online-adapted DNN predictions on the real system ($\epsilon = 1.5$).}
\label{fig:dnn_online}
\end{figure}

\noindent The proposed online update rule takes advantage of the dynamical nature of the DNN predicting online to achieve desirable performance under domain shift. This is a common problem in transfer learning and reinforcement learning, notably in sim2real transfer of control policies \cite{sim2real1}. Our method is summarized in Figure~\ref{fig:onnst_block_diag} for a sim2real control example.

\begin{figure*}[!t]
\centering
\includegraphics[width=6.5in]{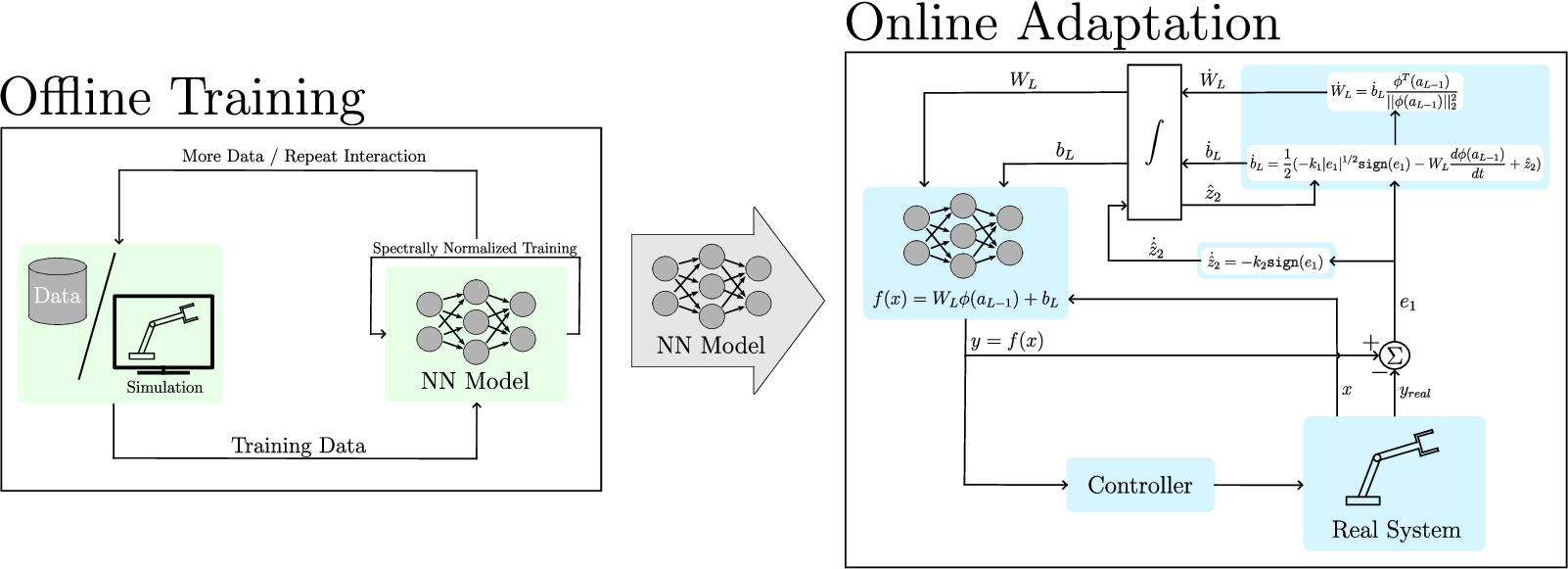}
\caption{Block diagram of the proposed online learning method, for a sim2real control example. The neural network model is first trained offline using conventional supervised or reinforcement learning under spectral normalization. The learned model, once deployed, is then updated online using the adaptation laws in Equations~\eqref{eq:b_update_st}, \eqref{eq:W_update_st}, and \eqref{eq:z_update_st}.}
\label{fig:onnst_block_diag}
\end{figure*}

\section{Related Work}

Neural networks have been studied extensively for use in control theory, beginning with notable works such as \cite{werbos_control} and \cite{narendra}. In most neural-network-based controllers, the neural network is simply used as an online adaptation instrument to guarantee a control objective (such as tracking error convergence) \cite{lewis96, sanner_older, chen_older, jag1, sahoo1, sahoo2, patil22}. However, while the mentioned works may be described as ``learning for controls," this work attempts to use controls for learning, utilizing the mathematical rigor and proven convergence properties of control theory to increase learning performance itself.

This work is inspired by the controllers derived in \cite{elkins_osc}, where a spacecraft attitude controller derived from pure reinforcement learning has no intrinsic stability guarantees, limiting its real-world application. This work also takes inspiration from \cite{neural_fly}, which combines offline DNN learning with an online adaptation law. \cite{neural_fly} specifically learns to compensate for the aerodynamic disturbances in quadcopter control, where the custom-design adversarial learning algorithm optimizes an invariant basis set for a Kalman-based online adaptation law. The works of \cite{neural_fly} and \cite{neural_lander} both use spectral normalization to bound the Lipschitz constant on the DNN for incorporation in a control law. 

 Fine-tuning DNN models by retraining lower layers is common in image processing and computer vision, such as in the works \cite{labonte, spottune, lastlayer1, lastlayer2} previously mentioned. Deep model reference adaptive control (MRAC) in \cite{deep_mrac} uses a similar offline DNN pretraining step with an online update rule based on concurrent learning, though the update rules and controller are specific to the linear MRAC case. The work in \cite{patil22} derives control-based update laws for both inner and outer layer weights for an arbitrary-depth DNN in an assumed system. Both \cite{patil22} and \cite{deep_mrac} utilize the projection operator to stabilize the DNN weight updates for the trajectory tracking error convergence control goal. The work in \cite{sun_lyap} uses a Lyapunov-based update law on the DNN outer layer, which is very similar to mean-squared-error-decreasing backpropagation. Further, DNN learning stability in \cite{sun_lyap} is guaranteed by a system control law. We aim to combine the ideas from these works to:

1) Improve \emph{learning} itself online, i.e. decreasing DNN approximation error over time, specifically under domain shift. This is contrasted to the decreasing trajectory tracking error of a control system by calculating a system input via some DNN output, as in the works listed above. While the proposed DNN update law can certainly be used inside an outer-loop controller (similar to \cite{neural_fly}), we aim to generalize to online DNN learning.

2) Stabilize the DNN outputs intrinsically, using control-theoretic updates to the DNN parameters itself.

3) Show how machine learning developments, such as spectral normalization, can give desirable control of prediction error convergence bounds in the proposed control-based online DNN update.

\section{Background}

In this section, we first detail the basics of DNN training and the feedforward DNN considered in this work. We then outline the theory and assumptions of the control algorithm used in this work, the super-twisting algorithm \cite{moreno_st}.

\subsection{DNN Basics}

This work considers a fully-connected, feedforward DNN $f(x) : \mathbb{R}^{n} \rightarrow \mathbb{R}^{m}$, with $L$ layers:

\begin{equation}
    \label{eq:relu_dnn}
    f(x) = W_{L} \phi(a_{L-1}) + b_L
\end{equation}

\noindent where $a_{L-1}$ is the output of the next-to-last layer, given as

\begin{equation}
    \label{eq:alm1}
    a_{L-1} = W_{L-1} \phi( W_{L-2} \phi( \cdots \phi( W_1 x + b_1)  \cdots ) + b_{L-2} ) + b_{L-1}
\end{equation}

\noindent where $W_i$ and $b_i$ is the weight and bias of the $i$\textsuperscript{th} layer, respectively; $\phi(\cdot)$ is a bounded nonlinear activation function, and $x \in \mathbb{R}^{n}$ is the neural network input. The DNN is assumed to be trained offline to optimize a feedback signal (such as minimizing error or maximizing reward) to be implemented for online inference, processing data and receiving error signals sequentially in time.

\subsection{The Super-Twisting Algorithm}

In model-based control of real-world systems, it is difficult to accurately model the real-world system being controlled. Control engineers have long studied how to design closed-loop controllers that render desirable performance under model uncertainty and discrepancy \cite{khalil, shtessel2014sliding}. One such robust control methodology is \emph{sliding mode control}, which aims to drive a nonlinear system to a desired manifold through the use of a discontinuous control signal \cite{spurgeon_smc}. The controller is typically designed for finite-time convergence of the sliding manifold to zero, which in turn is formulated for exponential convergence of error to zero in time. 

In conventional sliding mode control, the discontinuous control signal is undesirable, as it can introduce chatter and harm actuators in real-world systems. The \emph{super-twisting algorithm} (STA), introduced in \cite{levant_st}, is a sliding mode control algorithm that uses a continuous control signal to guarantee finitie-time convergence of the sliding variable, among other important qualities (cf. \cite{shtessel2014sliding}). In this work, we follow the STA formulation and Lyapunov stability proof given in \cite{moreno_st}.

Consider the STA system in state variable form

\begin{equation}
\label{eq:st_system}
\begin{split}
    \dot{x}_1 &= -k_1 |x_1|^{1/2} \texttt{sign}(x_1) + x_2 + p_1(x,t) \\
    \dot{x}_2 &= -k_2 \texttt{sign}(x_1) + p_2(x,t)
\end{split}
\end{equation}

\noindent where $x_1$, $x_2 \in \mathbb{R}$ are state variables, $k_1$, $k_2 \in \mathbb{R}^+$ are constant design gains, and $p_1(x,t)$, $p_2(x,t) : \mathbb{R}^2 \times \mathbb{R}_{\geq} \rightarrow \mathbb{R}$ are system perturbations. The right-hand side of Equation~\eqref{eq:st_system} is discontinuous, so the solution to the differential inclusion of \eqref{eq:st_system} is understood in the sense of Filippov \cite{filippov}. The STA is widely used in control, observation, and robust numerical differentiation \cite{shtessel2014sliding, levant_diff}.

Notably, the STA in Equation~\eqref{eq:st_system} is robust to the perturbation defined by

\begin{equation}
\label{eq:stable_perturb}
\begin{split}
    p_1(x, t) &= 0 \\
    |p_2(x, t)| &\leq D
\end{split}
\end{equation}

\noindent where $D \in \mathbb{R}^+$ is any constant, provided that the gains $k_1, k_2$ in Equation~\eqref{eq:st_system} are selected appropriately \cite{moreno_st, fridman1, levant1, levant2, levant3}. More precisely, Equation~\eqref{eq:st_system} under the perturbation defined in Equation~\eqref{eq:stable_perturb} establishes the origin $\{ x_1 = 0, x_2 = 0 \}$ as a global finite-time stable equilibrium point. This was proven geometrically in \cite{levant3}, using the homogeneity property of the controllers in \cite{levant2, orlov}, and via Lyapunov analysis in \cite{moreno_st}. The Lyapunov analysis in \cite{moreno_st} also establishes that, assuming the STA in \eqref{eq:st_system} is robust to the perturbation defined in \eqref{eq:stable_perturb}, then the STA is also robust to perturbations defined by 

\begin{equation}
\label{eq:stable_perturb2}
\begin{split}
    |p_1(x, t)| &\leq \delta_1 + \delta_2 \sqrt{|x_1| + x_2^2} \\
    |p_2(x, t)| &\leq D
\end{split}
\end{equation}

\noindent where $\delta_1, \delta_2 \in \mathbb{R}_{\geq}$ are constants. \cite{moreno_st} further shows that if $\delta_1 = 0$, $p_1(x, t)$ will vanish at the origin and the STA in \eqref{eq:st_system} will still converge to the origin in finite time. However, for perturbations that do not vanish at the origin in the $x_1$ channel (that is, $\delta_1 \neq 0$ such that $p_1(0, t) \neq 0$, the state trajectories will not generally converge to the origin, but will be globally ultimately bounded \cite{khalil}. This result is important, as in the DNN training development, the perturbation in the $x_1$ channel manifests from noise in the time derivative of the DNN input vector, from either numerical differentiation or measurement noise.

\section{Online DNN Updates Using Super-Twisting Control}


Assume an arbitrary function 2-times continuously differentiable with respect to time, $y$, which is to be approximated offline by an arbitrary-depth DNN as in Equation~\ref{eq:relu_dnn}. The system can be written in state space form as 

\begin{equation}
    \label{eq:ss_nom_sys}
    \begin{split}
        \dot{z}_1 &= z_2 \\
        \dot{z}_2 &= \ddot{y} \\
        y &= z_1 
    \end{split}
\end{equation}

\noindent where $z_1, z_2 \in \mathbb{R}^m$ are state variables. The DNN is trained offline on data generated from the ``nominal" system $y$ to some arbitrary approximation accuracy.

Next, we assume the trained DNN is implemented online to estimate the ``real" system, $y'$, which is a domain-shifted process of the nominal system, $y$, written as

\begin{equation}
    \label{eq:ss_real_sys}
    \begin{split}
        \dot{z}'_1 &= z'_2 \\
        \dot{z}'_2 &= \ddot{y}' \\
        y' &= z'_1 
    \end{split}
\end{equation}

\noindent where $z'_1, z'_2 \in \mathbb{R}^m$ are real-system state variables. The goal of this paper is to use control theory to find a provably-stable update law such that the DNN trained on the nominal system can perform effectively on the real system using feedback from an error signal. As stated above, this problem corresponds to domain shift and transfer learning problems common in using DNNs for reinforcement learning and forecasting, to name a few. In general, this problem is applicable to any DNN which is implemented to approximate a dynamical system online.

One possible method for the DNN-based model to learn online includes retraining the entire model on both the nominal and real system data at certain intervals during implementation, which can be inefficient as the training dataset grows, and typical backpropagation with gradient-based optimization does not inherently provide stability guarantees. Further, only retraining the model on the online real-system data acquired can induced ``catastrophic forgetting" \cite{online_learning_survey}, with the DNN parameters over-optimizing to the recent data. Taking inspiration from \cite{neural_fly}, we opt to only update the parameters of the last layer of the DNN, which is both efficient and preserves the overall feature representation of the DNN learned in the training data of the nominal system. In this way, the output of the next-to-last layer acts as a basis for function approximation of the last layer of the DNN. In \cite{neural_fly}, a custom adversarial learning algorithm is designed to maximize the independence of the learned basis output. In this work, we simply consider a DNN trained with conventional gradient-based optimization, due to its prevalence in modern AI/ML. We thus assume that the output of the next-to-last-layer, $\phi(a_{L-1})$, which can be considered as the learned representation of the important features of both $y$ and $y'$, is a suitable basis for approximation of the DNN's output layer.

To implement a controller for updating the last layer of the DNN online, we differentiate Equation~\eqref{eq:relu_dnn} with respect to time to get

\begin{equation}
    \label{eq:dnn_deriv}
    \dot{\hat{y}} = \frac{d f(x)}{dt} = \dot{W}_{L} \phi(a_{L-1}) + \dot{b}_L + \Gamma
\end{equation}

\noindent where $\hat{y} \in \mathbb{R}^m$ is the DNN output and $\Gamma = \Gamma(\dot{x}) = W_L \frac{d\phi(a_{L-1})}{dt}$ for notational simplicity. The control problem is to find update laws $\dot{W}_{L}$ and $\dot{b}_L$ to drive $e_1 = \hat{y} - y'  \rightarrow 0$ as $t \rightarrow \infty$.

Note that the activation derivative term $\frac{d\phi(a_{L-1})}{dt}$ term in $\Gamma$ is easily calculated using the chain rule, similar to backpropagation of error. The term $\Gamma$ is only a function of the time derivative of the DNN input vector, $\dot{x}$, since the weights and biases of the DNN are known. We will discuss two cases: (1) when $\dot{x}$ is known, and (2) when $\dot{x}$ is estimated or noisy. 

\subsection{Case I: Known $\dot{x}$}

The proposed online update laws for $W_L$, the output layer weight, and $b_L$, the output layer bias, are given as

\begin{equation}
    \label{eq:b_update_st}
    \dot{b}_L = \frac{1}{2}(-k_1 |e_1|^{1/2}\texttt{sign}(e_1) - \Gamma + \hat{z}_2)
\end{equation}

\begin{equation}
    \label{eq:W_update_st}
    \dot{W}_{L} = \dot{b}_L \frac{\phi^T(a_{L-1})}{||\phi(a_{L-1})||^2_2}
\end{equation}

\noindent where $\hat{z}_2$ is an augmented integral control term, evolved as

\begin{equation}
    \label{eq:z_update_st}
    \dot{\hat{z}}_2 = -k_2 \texttt{sign}(e_1)
\end{equation}

\noindent and $k_1, k_2 \in \mathbb{R}^+$ are again constant design gains. Note that in Equation~\eqref{eq:b_update_st}, $\Gamma = \Gamma(\dot{x})$ is used, since $\dot{x}$ is known in this case.

\begin{theorem}
Suppose the approximation target $y'$ is 2-times continuously differentiable, and its second time derivative is bounded such that $|\ddot{y}'| \leq D_y$. Suppose the the activation derivative term $\Gamma = W_L \frac{d\phi(a_{L-1})}{dt}$ is known. Further, suppose $y'$ is approximated by the DNN given in Equation~\eqref{eq:relu_dnn}, which is trained offline on data generated from $y$, and the NN input $x$ is continuously differentiable with respect to time. Then, for every $D_y > 0$, there exist design gains $k_1$, $k_2$ such that the last-layer update rules given in Equations~\eqref{eq:b_update_st}, \eqref{eq:W_update_st}, and \eqref{eq:z_update_st} cause $e_1 = 0$ to be a robustly, globally, finite-time stable equilibrium point.
\end{theorem}

\begin{proof}
Substituting the update laws in Equations~\eqref{eq:b_update_st}-\eqref{eq:z_update_st} into Equation~\eqref{eq:dnn_deriv}, we get the system

\begin{equation}
    \label{eq:relu_dnn_st_form}
    \begin{split}
        \dot{\hat{z}}_1 &= -k_1 |e_1|^{1/2}\texttt{sign}(e_1) + \hat{z}_2 \\
        \dot{\hat{z}}_2 &= -k_2 \texttt{sign}(e_1) \\
        \hat{y} &= \hat{z}_1 .
    \end{split}
\end{equation}

\noindent Subtracting the system in \eqref{eq:relu_dnn_st_form} by the real system given in \eqref{eq:ss_real_sys} gives the DNN error dynamics under the proposed online STA update given in Equations~\eqref{eq:b_update_st}-\eqref{eq:z_update_st}:

\begin{equation}
    \label{eq:error_dyn_st}
    \begin{split}
        \dot{e}_1 &= -k_1 |e_1|^{1/2}\texttt{sign}(e_1) + e_2 \\
        \dot{e}_2 &= -k_2 \texttt{sign}(e_1) - \ddot{y}' \\
    \end{split}
\end{equation}

\noindent where $e_2 = \hat{z}_2 - z'_2$. Since it is assumed the real system signal is bounded such that $|\ddot{y}'| \leq D_y$, the system given in \eqref{eq:error_dyn_st} is equivalent to the robust STA system in \eqref{eq:st_system} under the perturbation given in \eqref{eq:stable_perturb}, with $p_1(x,t) = 0$ and $|p_2(x,t)| = |\ddot{y}'| \leq D_y$.

The rest of the proof follows from the proof of Theorem 2 in the Appendix of \cite{moreno_st}.

\end{proof}

\subsection{Case II: Unknown or Estimated $\dot{x}$}

In some applications, the activation derivative term $\frac{d\phi(a_{L-1})}{dt}$ term is noisy and/or estimated. This can happen in some model reference adaptive control cases, specifically when the time derivative of the DNN input vector, $\dot{x}$, must be numerically differentiated for (or is calculated using noisy measurements). This can be seen by calculating $\frac{d\phi(a_{L-1})}{dt}$ for the DNN given in Equation~\eqref{eq:relu_dnn}:

\begin{equation}
    \label{eq:activ_deriv}
    \frac{d\phi(a_{L-1})}{dt} = \phi'(a_{L-1}) \odot W_{L-1} ( \phi'(a_{L-2}) \odot W_{L-2} ( \cdots \phi'(a_{1}) \odot W_1 \dot{x}))
\end{equation}

\noindent where $\odot$ denotes the Hadamard product, $a_i = W_i \phi(a_{i-1}) + b_i$ is the output of the $i$\textsuperscript{th} DNN layer, and $\phi'(z) = d\phi/dz$ denotes the activation function derivative. In Equation~\eqref{eq:activ_deriv}, the only potentially unknown term is $\dot{x}$, since the weights and biases of the DNN are known. Using Equation~\eqref{eq:activ_deriv}, we can expand $\Gamma = W_L \frac{d\phi(a_{L-1})}{dt}$ as

\begin{equation}
    \label{eq:gamma}
    \Gamma = W_L (\phi'(a_{L-1}) \odot W_{L-1} ( \phi'(a_{L-2}) \odot W_{L-2} ( \cdots  \phi'(a_{1}) \odot W_1 \dot{x}))) .
\end{equation}

\noindent Denoting the estimate of $\Gamma$ as $\hat{\Gamma} = \Gamma(\dot{\hat{x}})$, we can similarly write

\begin{equation}
    \label{eq:gamma_hat}
    \hat{\Gamma} = W_L (\phi'(a_{L-1}) \odot W_{L-1} ( \phi'(a_{L-2}) \odot W_{L-2} ( \cdots  \phi'(a_{1}) \odot W_1 \dot{\hat{x}})))
\end{equation}

\noindent where $\dot{\hat{x}}$ denotes the estimate of $\dot{x}$. The proposed STA online update rules are thus modified to use the estimate $\hat{\Gamma}$:

\begin{equation}
    \label{eq:b_update_st2}
    \dot{b}_L = \frac{1}{2}(-k_1 |e_1|^{1/2}\texttt{sign}(e_1) - \hat{\Gamma} + \hat{z}_2)
\end{equation}

\begin{equation}
    \label{eq:W_update_st2}
    \dot{W}_{L} = \dot{b}_L \frac{\phi^T(a_{L-1})}{||\phi(a_{L-1})||^2_2}
\end{equation}

\begin{equation}
    \label{eq:z_update_st2}
    \dot{\hat{z}}_2 = -k_2 \texttt{sign}(e_1)
\end{equation}

\noindent where $k_1, k_2 \in \mathbb{R}^+$ are again constant design gains. In this case, the estimate of $\hat{\Gamma}$ will introduce a perturbation term in the $x_1$ channel of the system in \eqref{eq:st_system}, which will cause the error system trajectories to not converge to zero but to be ultimately bounded.











    


Since it is desirable to decrease the error trajectory bound, we will explore how to quantify and control the bounds on $\Gamma$ and $\hat{\Gamma}$. The works \cite{neural_fly} and \cite{neural_lander} have shown that training DNNs under spectral normalization (SN) can be used to derive convenient stability guarantees when using DNNs in dynamical systems. Namely, spectral normalization controls the Lipschitz constant of a DNN, which defines the bound on the DNN output from a bounded input \cite{bartlett_sn}. \cite{neural_fly}, \cite{neural_lander}, and \cite{bartlett_sn} each show that SN training can improve out-of-domain generalization on DNN predictions and stabilizes training, especially in sensitive network structures such as GANs \cite{miyato_gan}.

A real function $g$ is Lipschitz continuous if there exists a constant $\gamma_{Lip} \in \mathbb{R}^+$ such that

\begin{equation}
    \label{eq:lipschitz}
    \frac{|| g(z_1) - g(z_2) ||_2}{|| z_1 - z_2 ||_2} \leq \gamma_{Lip}
\end{equation}

\noindent for any $z_1, z_2$ in the domain of $g$. The smallest constant $\gamma_{Lip}$ satisfying \eqref{eq:lipschitz}, $||g||_{Lip}$, is called the Lipschitz constant, which provides a convenient bound for relating function output given a bounded input. Following the analysis in \cite{neural_lander} and \cite{miyato_gan}, the Lipschitz constant of a function $g$ is the maximum singular value of its gradient in the domain, written as $||g||_{Lip} = \texttt{sup}_z \sigma (\nabla g(z))$, where $\sigma(\cdot)$ denotes the maximum singular value operator. The DNN in Equation~\eqref{eq:relu_dnn} is a composition of functions, recursively $h_i(z) = W_i z_{i-1} + b_i$ and $z_{i-1} = \phi(a_{i-1})$ for the $i$\textsuperscript{th} layer. Further, the Lipschitz constant for a composition of functions $g_L \circ g_{L-1} \circ \cdots \circ g_1$ is bounded by the inequality

\begin{equation}
    \label{eq:lipschitz_ineq}
    || g_L \circ g_{L-1} \circ \cdots \circ g_1 ||_{Lip} \leq || g_L ||_{Lip} || g_{L-1} ||_{Lip} \cdots || g_{1} ||_{Lip} .
\end{equation}

\noindent Since $||g||_{Lip} = \texttt{sup}_z \sigma (\nabla g(z))$,  the Lipschitz constant of the $i$\textsuperscript{th} DNN layer is $||h_i||_{Lip} = \texttt{sup}_{z_{i-1}} \sigma (\nabla (W_i z_{i-1} + b_i)) = \texttt{sup}_{z_{i-1}} \sigma(W_i) = \sigma(W_i)$. That is, the Lipschitz constant of the $i$\textsuperscript{th} DNN layer is the maximum singular value of the weight matrix $W_i$. This can be calculated by $\sigma(W_i) = \texttt{max}(\sqrt{\lambda(W_i^T W_i)})$, where $\lambda(\cdot)$ is the eigenvalue operator. Using the inequality in \eqref{eq:lipschitz_ineq}, the Lipschitz constant for the DNN in Equation~\eqref{eq:relu_dnn} can be bounded by

\begin{equation}
    \label{eq:dnn_lips}
    || f(x) ||_{Lip} \leq \sigma(W_L) || \phi ||_{Lip} \sigma(W_{L-1}) || \phi ||_{Lip} \cdots \sigma(W_{1})
\end{equation}

\noindent where the Lipschitz constant of the activation functions $|| \phi ||_{Lip}$ can be easily found based on the activation function used in the DNN. The ReLU activation function is defined as $\phi(a_i) = \texttt{max}(0, a_i)$. Its derivative can be written as 

\begin{equation}
\label{eq:relu_deriv}
{\phi'(a_i)} =
\begin{cases}
1&{\text{if}}\ a_i \geq 0,\\
{0}&{\text{otherwise}}
\end{cases}
\end{equation}

\noindent which shows a Lipschitz constant $|| \phi ||_{Lip} = 1$,. We can thus bound the Lipschitz constant of the DNN in Equation~\eqref{eq:relu_dnn} with ReLU activations by

\begin{equation}
    \label{eq:dnn_lips_relu}
    || f(x) ||_{Lip} \leq \prod_{i=1}^L \sigma(W_i)
\end{equation}

\noindent which follows directly from \eqref{eq:dnn_lips}. To control $|| f(x) ||_{Lip}$ during training, we implement Algorithm~\ref{alg:alg1}, where $\gamma$ is the desired upper bound on the Lipschitz constant of the DNN in Equation~\eqref{eq:relu_dnn} with ReLU activations, such that $|| f(x) ||_{Lip} \leq \gamma$ \cite{neural_lander}.

\begin{algorithm}[H]
\caption{Spectrally normalized ReLU DNN training.}\label{alg:alg1}
\begin{algorithmic}
\STATE 

\textbf{for} epoch \textbf{in} range(max\_epochs)

\STATE \hspace{0.5cm}$\textbf{optimize NN parameters } W_i, b_i \textbf{ for } i = 1,2, \cdots, L$
\STATE \hspace{0.5cm}$\textbf{for } i = 1, 2, \cdots, L:$
\STATE \hspace{0.5cm}\hspace{0.5cm}$\sigma_i = \sigma(W_i)$
\STATE \hspace{0.5cm}\hspace{0.5cm}$\textbf{if } \sigma_i > \gamma^{1/L}:$
\STATE \hspace{0.5cm}\hspace{0.5cm}\hspace{0.5cm}$W_i \gets \frac{W_i}{\sigma_i} \gamma^{1/L}$
\STATE \hspace{0.5cm}\hspace{0.5cm}$\textbf{else: }$
\STATE \hspace{0.5cm}\hspace{0.5cm}\hspace{0.5cm}$\textbf{continue }$

\end{algorithmic}
\end{algorithm}

Note that in Algorithm~\ref{alg:alg1}, the normalizing weight update $W_i \gets \frac{W_i}{\sigma_i} \gamma^{1/L}$ upper bounds the Lipschitz constant of the $i$\textsuperscript{th} DNN layer $h_i$ by $||h_i||_{Lip} \leq \gamma^{1/L}$.






\begin{lemma}\label{lemma:gamma_bound}
Consider the DNN given in Equation~\eqref{eq:relu_dnn}, with activation functions $\phi(\cdot)$ defined as the rectified linear unit (ReLU) function, trained with spectral normalization via Algorithm~\ref{alg:alg1}. If the Lipschitz constant of the DNN is upper bounded by $|| f(x) ||_{Lip} \leq \gamma$, then the Lipschitz constant of $\Gamma$ is also upper bounded by $|| \Gamma ||_{Lip} \leq \gamma$.
\end{lemma}

\begin{proof}
From Equation~\eqref{eq:gamma}, it can be seen that $\Gamma$ is a recursion of linear operators, similar to the analysis above for the ReLU DNN itself. Equation~\eqref{eq:gamma} can be rewritten as a recursive composition of functions as

\begin{equation}
    \label{eq:gamma2}
    \Gamma = h'_L \circ z'_{L-1} \circ h'_{L-1} \circ \cdots \circ z'_{1} \circ h'_1
\end{equation}

\noindent where $h'_i = W_i z'_{i-1}$ and $z'_{i-1} = \phi'(a_{i-1}) \odot h'_{i-1}$. The Lipschitz constant of a linear map is its maximum singular value of its gradient, giving $||h'_i||_{Lip} = \sigma(W_i)$. The function $z'_{i-1}$ can be rewritten as the linear map $z'_{i-1} = \texttt{diag}(\phi'(a_{i-1})) h'_{i-1}$ such that $||z'_{i-1}||_{Lip} = \sigma(\texttt{diag}(\phi'(a_{i-1})))$. The maximum singular value of a diagonal matrix is equivalent to the infinity norm of its diagonal. With the ReLU derivative given in Equation~\eqref{eq:relu_deriv}, $||z'_{i-1}||_{Lip} = ||\phi'(a_{i-1})||_{\infty} = 1$. From the inequality in \eqref{eq:lipschitz_ineq}, the Lipschitz constant of $\Gamma$ is thus upper bounded by

\begin{equation}
    \label{eq:gamma_lip_bound}
    ||\Gamma||_{Lip} \leq \sigma(W_i) \sigma(W_{i-1}) \cdots \sigma(W_1) = \prod_{i=1}^L \sigma(W_i) .
\end{equation}

\noindent Using Algorithm~\ref{alg:alg1} to train the DNN controls the singular values of the DNN weights, such that $\prod_{i=1}^L \sigma(W_i) = \gamma$. Thus, the Lipschitz constant of $\Gamma$ is upper bounded by

\begin{equation}
    \label{eq:gamma_lip_bound_last}
    ||\Gamma||_{Lip} \leq \gamma .
\end{equation}

\end{proof}

\begin{theorem}\label{thm2}
Suppose the approximation target $y'$ is 2-times continuously differentiable, and its second time derivative is bounded such that $|\ddot{y}'| \leq D_y$. Suppose the the activation derivative term $\Gamma = W_L \frac{d\phi(a_{L-1})}{dt}$ is estimated by $\hat{\Gamma}$. Further, suppose $y'$ is approximated by the DNN given in Equation~\eqref{eq:relu_dnn} with ReLU activation functions, which is trained offline on data generated from $y$ via Algorithm~\ref{alg:alg1}, and the NN input $x$ is continuously differentiable with respect to time. Then, for every $D_y > 0$, there exist design gains $k_1$, $k_2$ such that the last-layer update rules given in Equations~\eqref{eq:b_update_st2}, \eqref{eq:W_update_st2}, and \eqref{eq:z_update_st2} cause the error trajectories $e_1, e_2$ to be globally ultimately bounded for a small enough $\delta_2$ in Equation~\eqref{eq:stable_perturb2}.
\end{theorem}

\begin{proof}

Substituting the update laws in Equations~\eqref{eq:b_update_st2}-\eqref{eq:z_update_st2} into Equation~\eqref{eq:dnn_deriv}, we now get the system

\begin{equation}
    \label{eq:relu_dnn_st_form2}
    \begin{split}
        \dot{\hat{z}}_1 &= -k_1 |e_1|^{1/2}\texttt{sign}(e_1) + \hat{z}_2 + \Gamma - \hat{\Gamma} \\
        \dot{\hat{z}}_2 &= -k_2 \texttt{sign}(e_1) \\
        \hat{y} &= \hat{z}_1 .
    \end{split}
\end{equation}

\noindent Subtracting the system in \eqref{eq:relu_dnn_st_form2} by the real system given in \eqref{eq:ss_real_sys} gives the DNN error dynamics under the modified online STA update, given in Equations~\eqref{eq:b_update_st2}-\eqref{eq:z_update_st2}:

\begin{equation}
    \label{eq:error_dyn_st2}
    \begin{split}
        \dot{e}_1 &= -k_1 |e_1|^{1/2}\texttt{sign}(e_1) + e_2 + \Gamma - \hat{\Gamma} \\
        \dot{e}_2 &= -k_2 \texttt{sign}(e_1) - \ddot{y}'.
    \end{split}
\end{equation}

\noindent Since it is assumed $|\ddot{y}'| \leq D_y$, the system given in \eqref{eq:error_dyn_st} is equivalent to the STA system in \eqref{eq:st_system} with perturbation terms $p_1(x,t) = \Gamma - \hat{\Gamma}$ and $p_2(x,t) = -\ddot{y}'$. The perturbation in the $e_2$ channel, $p_2$, is assumed to be upper-bounded by $|p_2(x,t)| = |\ddot{y}'| \leq D_y$.

Investigating the perturbation in the $e_1$ channel, $p_1(x,t) = \Gamma - \hat{\Gamma}$, we can write

\begin{equation}
    \label{eq:p1_proof2}
    \frac{||\Gamma - \hat{\Gamma}||_2}{||\dot{x} - \dot{\hat{x}}||_2} \leq ||\Gamma||_{Lip} 
\end{equation}

\noindent directly from the definition of Lipschitz continuity in \eqref{eq:lipschitz}. Using Lemma~\ref{lemma:gamma_bound} and rearranging, $p_1(x, t)$ is upper bounded by

\begin{equation}
    \label{eq:p1_bound_proof2}
    ||p_1(x, t)||_2 \leq \gamma ||\dot{x} - \dot{\hat{x}}||_2 .
\end{equation}

\noindent That is, the disturbance in the $e_1$ channel is bounded by the desired Lipschitz bound on the DNN times the error in the time derivative of the DNN input vector. Thus, in the presence of noise in $\dot{x}$ via measurements or numerical differentiation, training the DNN with spectral normalization can give the control designer quantifiable upper bounds on trajectory convergence for the proposed online last-layer update rules in Equations~\eqref{eq:b_update_st2}, \eqref{eq:W_update_st2}, and \eqref{eq:z_update_st2}.

The rest of the proof follows from the proof of Theorem 3 in the Appendix of \cite{moreno_st}. 

\end{proof}

\section{Simulation Examples}

\textit{Case I:} Consider the Van der Pol system in \eqref{eq:vanderpol}, with the training (or ``nominal") system value $\epsilon = 1$, assuming the dynamics $\dot{\theta} = z$ are known. The DNN in Equation~\eqref{eq:relu_dnn} with 4 layers of 8, 16, 8, 1 neurons, respectively, is trained with ReLU activation functions and spectral normalization via Algorithm~\ref{alg:alg1}, with $\gamma = 32$. The training dataset of features $\mathcal{X} = \{z, \theta\}$ and label dataset of $\mathcal{Y} = \{ \dot{z} \}$ are generated by the nominal system for 30 seconds at time intervals of 0.01 seconds. The DNN is trained with the Adam optimizer for 20,000 epochs with a batch size of 8. The DNN is then implemented online onto the real system of $\epsilon = 1.5$, where the time derivative of the DNN input vector $\dot{x} = [\hat{y}, z]^T$ is used. Note that this input vector uses the DNN prediction $\hat{y}$, which approximates the learning target $\dot{z}$. The update rules of Equations~\eqref{eq:b_update_st}, \eqref{eq:W_update_st}, and \eqref{eq:z_update_st} are used, with $k_1 = 50$, $k_2 = 1$.

\textit{Case II.a:} The Van der Pol system in \eqref{eq:vanderpol} is again considered, with the training (or ``nominal") system value $\epsilon = 1$, assuming the dynamics of $\theta$ are known but noisy. The DNN in Equation~\eqref{eq:relu_dnn} with 4 layers of 8, 16, 8, 1 neurons, respectively, is trained with ReLU activation functions and spectral normalization via Algorithm~\ref{alg:alg1}, with $\gamma = 1$. The training dataset of features $\mathcal{X} = \{z, \theta\}$ and label dataset of $\mathcal{Y} = \{ \dot{z} \}$ is generated by the nominal system for 30 seconds at time intervals of 0.01 seconds. The DNN is trained with the Adam optimizer for 20,000 epochs with a batch size of 8. The DNN is then implemented online onto the real system of $\epsilon = 1.5$, where the time derivative of the DNN input vector $\dot{x} = [\hat{y}, z + 10 \sin{(20 \pi t)}]^T$ is used. The update rules of Equations~\eqref{eq:b_update_st2}, \eqref{eq:W_update_st2}, and \eqref{eq:z_update_st2} are used, with $k_1 = 50$, $k_2 = 1$.

\textit{Case II.b:} Case II.b is the same as Case II.a above, but the DNN is not trained with spectral normalization. The DNN is implemented online to predict the real system, where update rules of Equations~\eqref{eq:b_update_st2}, \eqref{eq:W_update_st2}, and \eqref{eq:z_update_st2} are used, with $k_1 = 50$, $k_2 = 1$.

The DNN outputs over time for each case, compared to the real system, are shown in Figure~\ref{fig:dnn_sys_full}. To compare the behavior of each case, the prediction error $e_1$ for each case is shown over time in Figure~\ref{fig:dnn_errors}. To compare the error analysis in Theorem~\ref{thm2}, $||\Gamma - \hat{\Gamma}||_2$ for Case II.a and Case II.b are plotted over time, along with the defined error bound in Equation~\eqref{eq:gamma_lip_bound_last}, in Figure~\ref{fig:dnn_gammas}.

\begin{figure}[!htbp]
\centering
\includegraphics[width=3in]{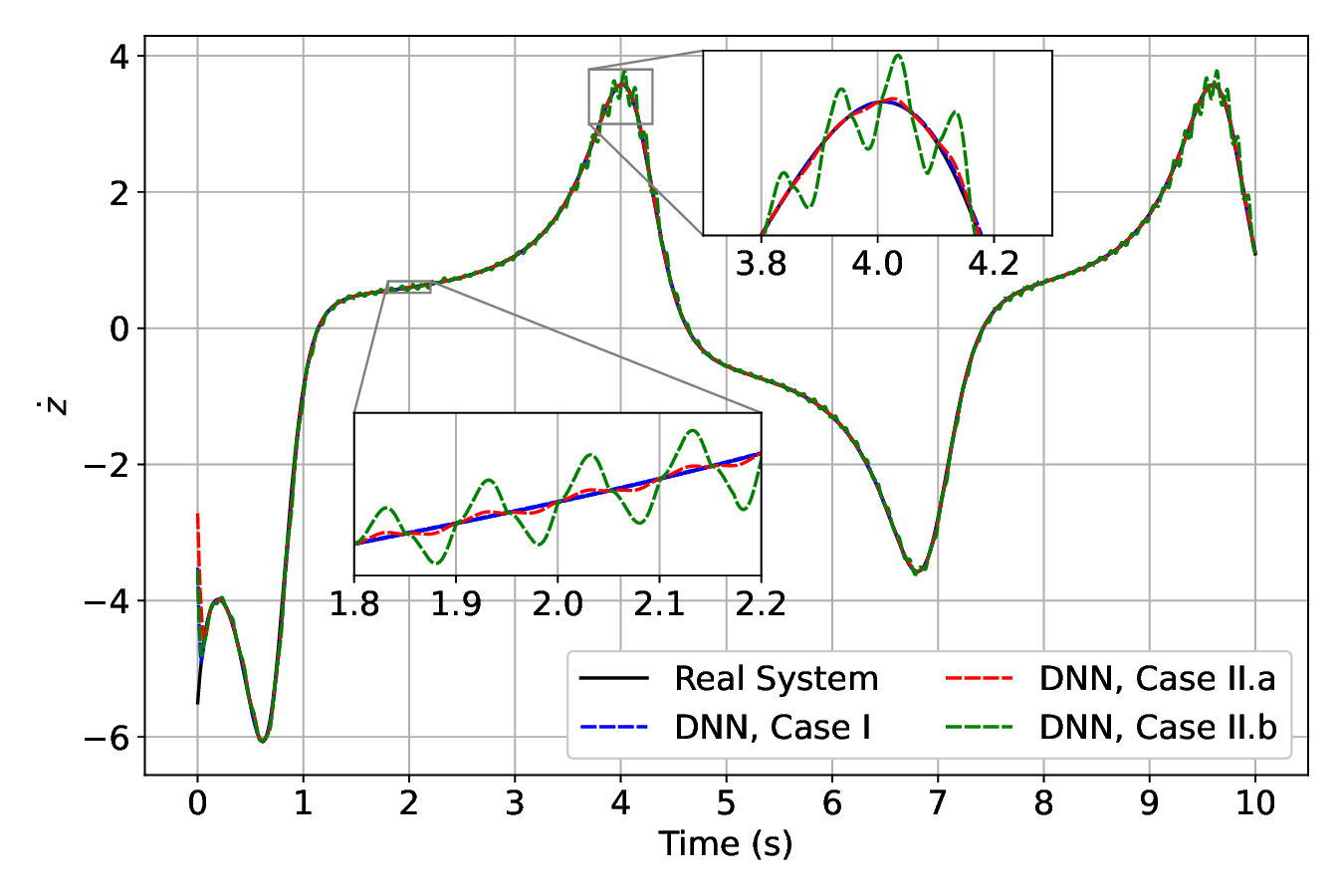}
\caption{Online-adapted DNN predictions on the real system ($\epsilon = 1.5$) for each case.}
\label{fig:dnn_sys_full}
\end{figure}

\begin{figure}[!htbp]
\centering
\includegraphics[width=3in]{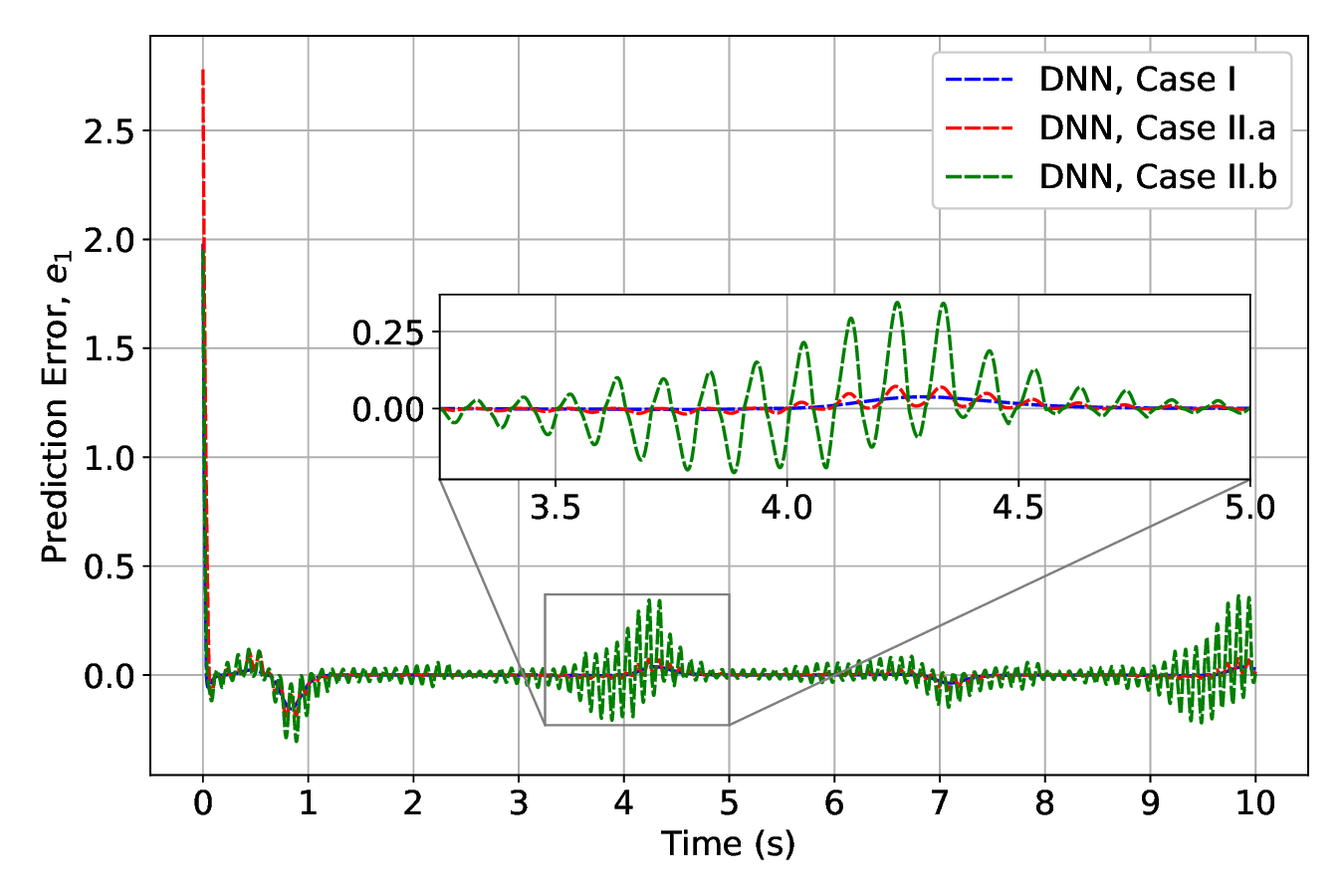}
\caption{DNN prediction error on the real system ($\epsilon = 1.5$) for each case.}
\label{fig:dnn_errors}
\end{figure}

\begin{figure}[!htbp]
\centering
\includegraphics[width=3in]{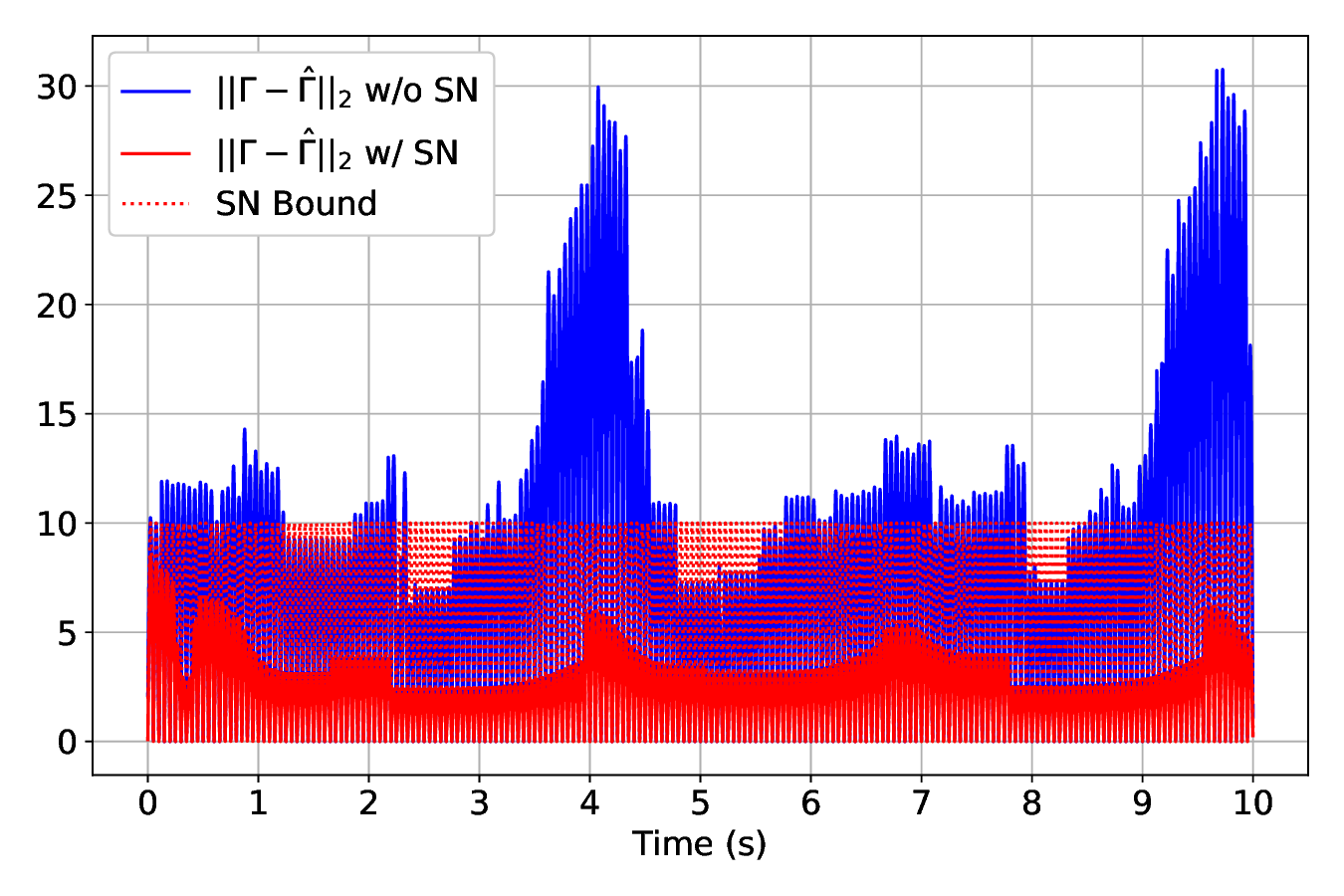}
\caption{Perturbation analysis on the real system ($\epsilon = 1.5$) for the DNN trained with SN (Case II.a) and the DNN trained without SN (Case II.b). The red dotted line represents the upper bound from \eqref{eq:gamma_lip_bound} at each timestep.}
\label{fig:dnn_gammas}
\end{figure}

In Figure~\ref{fig:dnn_sys_full}, each DNN successfully converges to near the real system using the proposed online update rules, after being trained on a small nominal system dataset. In Figure~\ref{fig:dnn_errors}, the perturbation in the $e_1$ channel from the disturbance in $\Gamma - \hat{\Gamma}$ now causes the $e_1$ trajectories in Case II.a and II.b to be ultimately bounded. However, when the DNN is trained with spectral normalization, the controlled Lipschitz constant upper bounds the disturbance $\Gamma - \hat{\Gamma}$, causing the upper bound of the $e_1$ trajectory to be smaller than the DNN trained without spectral normalization in Case II.b. In Figure~\ref{fig:dnn_gammas}, the upper bound on $\Gamma - \hat{\Gamma}$ from \eqref{eq:gamma_lip_bound} is plotted over time, along with the $\Gamma - \hat{\Gamma}$ disturbance from both Case II.a (w/ SN) and Case II.b (w/o SN). It can be seen from this figure that the DNN trained without spectral normalization causes a larger disturbance in the $e_1$ channel, where the DNN trained with spectral normalization obeys the proven disturbance bound in \eqref{eq:gamma_lip_bound}.

\section{Conclusion}

In this work, we study the use of control theoretic techniques for evolving deep neural networks (DNNs) online during inference, when the online distribution is shifted from the offline distribution used to train the DNN. We specifically consider when DNNs to approximate dynamical systems, which is of particular interest to transfer learning and reinforcement learning for forecasting, controls, and sim2real transfer. When the DNN is used to approximate a dynamical system online, evolving the itself DNN can be considered as a dynamical system to be controlled. We use the super-testing algorithm, a well-known sliding mode control algorithm, to evolve the last layer parameters of the DNN in continuous time. Last-layer evolution has been shown in transfer learning to effectively and efficiently provide updates to models from newly-acquired data while limiting catastrophic forgetting, where retraining the entire model can be expensive. Under our proposed online update rules, we give proofs of error trajectory convergence of the DNN outputs, which is desirable in many online DNN prediction scenarios, since conventional gradient-based backpropagation does not intrinsically provide any DNN performance guarantees or bounds. Further, when the time derivative of the DNN input vector is noisy, we show that training the DNN under spectral normalization can improve online DNN prediction performance, since the desired spectral normalization constant of the DNN upper bounds the prediction error trajectory convergence. We validate the proposed methodology for each described case in numerical simulation, with the DNN trained to approximate the dynamics of the Van der Pol oscillator online under domain shift. In each case, we show that our proposed method performs effectively online on the domain-shifted system and that the bounds given in the spectral normalization proofs do hold in simulation.

\section{Funding Acknowledgement}
Jacob G. Elkins is supported by the Department of Defense (DoD) through the National Defense Science and Engineering Graduate (NDSEG) Fellowship Program.



\end{document}